\documentclass[a4paper,aps,pra,twocolumn,10pt,superscriptaddress]{revtex4-1}

\usepackage[utf8]{inputenc}
\usepackage[T1]{fontenc}
\usepackage[sc,osf]{mathpazo}
\usepackage{amsmath, amsthm, amssymb}
\usepackage{graphicx}
\usepackage{dcolumn}
\usepackage{bm}
\usepackage{bbm}
\usepackage{hyperref}
\usepackage{tikz}

\usetikzlibrary{arrows}
\usetikzlibrary{patterns}
\usetikzlibrary{decorations.pathmorphing}

\def\ci{\perp\!\!\!\perp}

\tikzstyle{var}=[circle,draw=black,fill=white,thick,minimum size=20pt,inner sep=0pt]
\tikzstyle{arr}=[->,>=stealth',draw=black,line width=1pt]

\newtheorem{prop}{Proposition}
\newtheorem{corollary}[prop]{Corollary}
\newtheorem{theorem}[prop]{Theorem}

\newtheorem{definition}[prop]{Definition}

\newtheorem{lemma}[prop]{Lemma}

\newcommand{\beq}{\begin{equation}}
\newcommand{\eeq}{\end{equation}}
\newcommand{\bea}[1]{\begin{equation}\begin{array}{#1}}
\newcommand{\eea}{\end{array}\end{equation}}
\newcommand{\beqn}{\begin{eqnarray}}
\newcommand{\eeqn}{\end{eqnarray}}

\renewcommand{\rho}{\varrho}

\newcommand{\processnext}[1]{%
  \ifx\listfinish#1\empty\else\listact{#1}\expandafter\processnext\fi}

\begin{document}

\title{Inferring latent structures via information inequalities}

\author{R. Chaves}
\affiliation{Institute for Physics, University of Freiburg, Rheinstrasse 10, D-79104 Freiburg, Germany}
\author{L. Luft}
\affiliation{Institute for Physics, University of Freiburg, Rheinstrasse 10, D-79104 Freiburg, Germany}
\author{T. O. Maciel}
\affiliation{Institute for Physics, University of Freiburg, Rheinstrasse 10, D-79104 Freiburg, Germany}
\affiliation{Physics Department, Federal University of Minas Gerais, Brazil}
\author{D. Gross}
\affiliation{Institute for Physics, University of Freiburg, Rheinstrasse 10, D-79104 Freiburg, Germany}
\affiliation{Freiburg Center for Data Analysis and Modeling, Germany}
\author{D. Janzing}
\affiliation{Max Planck Institute for Intelligent Systems, T\"ubingen, Germany}
\author{B. Sch\"olkopf}
\affiliation{Max Planck Institute for Intelligent Systems, T\"ubingen, Germany}

\begin{abstract}
One of the goals of probabilistic inference is to decide whether an
  empirically observed distribution is compatible with a candidate
  Bayesian network. However, Bayesian networks with hidden variables
  give rise to highly non-trivial constraints on the observed
  distribution. Here, we propose an information-theoretic approach,
  based on the insight that conditions on \emph{entropies} of Bayesian
  networks take the form of simple linear inequalities.  We describe
  an algorithm for deriving entropic tests for latent structures.  The
  well-known conditional independence tests appear as a special case.
  While the approach applies for generic Bayesian networks, we
  presently adopt the causal view, and show the versatility of the
  framework by treating several relevant problems from that domain:
  detecting common ancestors, quantifying the strength of causal
  influence, and inferring the direction of causation from
  two-variable marginals.
\end{abstract}

\maketitle

\section{Introduction}
Inferring causal relationships
from empirical data is one of the
prime goals of science.
A common scenario reads as follows: Given $n$ random variables $X_1,\dots,X_n$, infer
their causal relations from
a list of $n$-tuples i.i.d. drawn from $P(X_1,\dots,X_n)$.
To formalize causal relations, it has become popular to
use
directed acyclic graphs (DAGs) with random variables as nodes
(c.f.\
Fig.~\ref{fig:DAGS_instrumental})
and arrows meaning direct causal influence
\cite{Pearlbook,Spirtesbook}.
Such causal models  have been called {\it causal} Bayesian networks
\cite{Pearlbook},
as opposed to traditional Bayesian networks
that formalize conditional independence relations without having necessarily
a causal interpretation.
One of the tasks of
causal inference is to decide which causal Bayesian networks are
compatible with empirically observed data.

The most common way to infer the set of possible DAGs from
observations is based on the \emph{Markov condition} (c.f.\
Sect.~\ref{sec:info_approach})  stating which conditional statistical
independencies are implied by the
graph structure, and the \emph{faithfulness assumption} stating that
the joint distribution is generic
for the DAG in the sense that no additional
independencies hold \cite{Spirtesbook,Pearlbook}.
Causal inference via Markov condition and faithfulness
has been well-studied for the case where all variables  are
observable,
but some work also refers to
latent structures where only a subset is observable
\cite{Pearlbook,Richardson2002,Ali2005}.
In that case, we
are faced with the problem of characterizing the set
of \emph{marginal distributions} a given Bayesian network can give
rise to. If an observed distribution lies outside the set of marginals
of a candidate network, then that model can be rejected as an
explanation of the data.  Unfortunately, it is widely appreciated that
Bayesian networks involving latent variables impose highly
non-trivial
constraints on the
distributions compatible with it
\cite{Tian2002,Steeg2011,Kang2012a,Kang2012b}.

These technical difficulties stem from the fact that the conditional independencies amount to non-trivial algebraic conditions on probabilities.  More precisely,
the marginal regions are semi-algebraic sets that can, in principle,
be characterized by a finite number of polynomial equalities and
inequalities \cite{Geiger1998}. However, it seems that in
practice, algebraic statistics is still limited to very simple models.

In order to circumvent this problem, we propose an
information-theoretic approach for causal inference.  It is based on
an entropic framework for treating marginal problems that, perhaps
surprisingly, has recently been introduced in the context of Bell's
Theorem and the foundations of quantum mechanics
\cite{FritzChaves2012,Chaves2013b}.  The basic insight is that the \emph{algebraic} condition $p(x,y)=p_1(x)
p_2(y)$ for independence becomes a \emph{linear} relation
$H(X,Y)=H(X)+H(Y)$ on the level of entropies.  This opens up the
possibility of using computational tools such as linear programming to
find marginal constraints -- which contrasts pleasantly with the
complexity of algebraic methods that would otherwise be necessary.

\subsection{Results}
Our main message is that a significant amount of information about
causation is contained in the entropies of observable
variables and that there are relatively simple and systematic ways of
unlocking that information.
We will make that case by discussing a great variety of applications,
which we briefly summarize here.

After
introducing the geometric and algorithmic framework in
Sections~\ref{sec:info_approach} \& \ref{sec:algorithm},
we start with the applications in
Section~\ref{subsec:instrumental} which treats instrumentality tests.
There, we argue that the non-linear nature of entropy, together with
the fact that it is agnostic about the number of outcomes of a random
variable, can greatly reduce the complexity of causal tests.

Two points are made in Sec.~\ref{sec:bad_statistics}, treating an
example where the direction of causation between a set of variables is
to be inferred. Firstly, that marginal entropies of few
variables can carry non-trivial information about conditional
independencies encoded in a larger number of variables.
This may have practical and statistical advantages. Secondly, we
point out applications to tests for quantum non-locality.

In Sec.~\ref{sec:ancestors} we consider the problem of distinguishing between different hidden common ancestors causal structures.
While most of the entropic tests in this
paper have been derived using automated linear programming algorithms,
this section presents analytic proofs valid for any
number of variables.

Finally, Sec.~\ref{sec:quant_causal} details
three conceptually important realizations: (1) The framework can
be employed to derive quantitative lower bounds on the strength of
causation between variables. (2) The degree of violation of entropic
inequalities carries an operational meaning.
(3) Under some assumptions, we can exhibit novel conditions for
distinguishing dependencies created through common ancestors from
direct causation.

\section{The information-theoretic description of Bayesian networks}
\label{sec:info_approach}

In this section we introduce the basic technical concepts that are
required to make the present paper self-contained. More details can be
found in \cite{Pearlbook,FritzChaves2012,Chaves2013b}.

\subsection{Bayesian networks}

Here and in the following, we will consider $n$ jointly distributed
discrete random variables $(X_1, \dots, X_n)$. Uppercase letters label
random variables while lowercase label the values taken by these
variables, e.g.\, $p(X_i=x_i,X_j=x_j) \equiv p(x_i,x_j)$.

Choose a \emph{directed acyclic graph} (DAG) which has the $X_i$'s as
its vertices. The $X_i$'s form a \emph{Bayesian network} with respect to
the graph if every variable can be expressed as
a function of its
parents $PA_i$ and an unobserved noise term $N_i$, such that the $N_i$'s are
jointly independent.
%
That is the case if and only if the distribution is of the form
\begin{equation*}
	p(x) = \prod_{i=1}^n p (x_i | \mathrm{pa}_{i} ).
\end{equation*}
Importantly, this is
equivalent to demanding that the $X_i$ fulfill the \emph{local Markov
property}: Every $X_i$ is conditionally independent of its
non-descendants $ND_i$ given its parents $PA_i$: $X_i \ci  ND_i | PA_i$.

We allow some of the nodes in the DAG to stand for \emph{hidden variables} that are not directly observable.
Thus, the marginal distribution of the observed variables becomes
\begin{equation}
p(v)= \sum_{u} \prod_{i=1,\dots,m} p(v_i \vert \mathrm{pa}_{i})
\prod_{j=1,\dots,n-m} p(u_j \vert \mathrm{pa}_{j}),
\end{equation}
where $V=(V_1,\dots,V_m)$ are the observable variables and
$U=(U_1,\dots,U_{n-m})$
the hidden ones.

\subsection{Shannon Entropy cones}
Again, we
consider a collection of $n$ discrete random variables $X_1, \dots,
X_n$.
We denote the set of indices
of the random variables by
$[n]=\{1, \dots, n\}$ and its power set (i.e., the set of subsets) by
$2^{[n]}$. For every subset $S\in 2^{[n]}$ of indices, let $X_S$ be
the random vector $(X_i)_{i\in S}$ and denote by $H(S):=H(X_S)$ the
associated
Shannon entropy given by $H(X_S)=-\sum_{x_s}p(x_s)\log_2 p(x_s)$.
With this convention, entropy becomes a function
\begin{eqnarray*}
	H: 2^{[n]} \to \mathbbm{R}, \qquad S \mapsto H(S)
\end{eqnarray*}
on the power set. The linear space of all set functions
will be denoted by
$R_n$.  For every function $h\in R_n$ and $S\in2^{[n]}$, we
use the notations $h(S)$ and $h_S$ interchangeably.

The region
\begin{eqnarray*}
	\left\{ h \in R_n \,|\, h_S = H(S) \text{ for some entropy function
	} H \right\}
\end{eqnarray*}
of vectors in $R_n$ that correspond to entropies has been studied
extensively in information theory \cite{Yeung2008}. Its closure is
known to be a convex cone, but a tight and explicit description
is unknown.
However, there is a standard outer approximation which is
the basis of our work: the
\emph{Shannon cone}
$\Gamma_n$.
The
Shannon cone is the
polyhedral closed convex cone of set functions $h$
that respect the following set of linear inequalities:
\begin{eqnarray}
	h_{}([n]\setminus\{i\}) &\leq& h_{}([n])
	\label{shannonineqs_basic}
	\label{eqn:monotonicity} \\
	h_{}(S) + h_{}(S\cup\{i,j\}) &\leq& h_{}(S\cup \{i\}) + h_{}(S\cup \{j\})
	\nonumber 
	\\
	h_{}(\emptyset) &=& 0
	\nonumber 
\end{eqnarray}
%
for all $S \subset [n] \setminus\{i,j\}$, $i \neq j$ and $i, j\in
[n]$. These inequalities hold for entropy: The first relation  -- known as
 \emph{monotonicity}  --
states that the uncertainty about a
set of variables should always be larger than or equal to the
uncertainty about any subset of it. The second inequality
is the \emph{sub-modularity} condition which is
equivalent to
the positivity of the conditional mutual information  $I(X_i:X_j
\vert X_S)= H(X_{S \cup i}) +H(X_{ S \cup j})-H( X_{S
\cup\{i,j\}})-H(X_{S}) \geq 0$. The inequalities
above
are known as the \emph{elementary
inequalities} in information theory or the \emph{polymatroidal
axioms} in combinatorial optimization.
An inequality that follows from
the elementary ones
is said to be of
\emph{Shannon-type}.

The elementary inequalities encode the constraints that the entropies
of \emph{any} set of random variables are subject to.  If one further
demands that the random variables are a Bayesian network with respect
to some given DAG, additional relations between their entropies will
ensue. Indeed, it is a straight-forward but central realization for
the program pursued here, that CI relations
%
faithfully
translate to homogeneous linear constraints on entropy:
\begin{equation}\label{eqn:entropicmarkov}
	X \ci Y | Z
	\qquad
	\Leftrightarrow
	\qquad
	I(X : Y | Z) = 0.
\end{equation}
The conditional independencies (CI) given by the local Markov condition are sufficient to characterize distributions
that form a Bayesian network w.r.t.\ some fixed DAG. Any such
distribution exhibits further CI relations, which can be
algorithmically enumerated using the so-called \emph{$d$-separation
criterion} \cite{Pearlbook}.
Let $\Gamma_c$ be the subspace of $R_n$ defined by the equality
(\ref{eqn:entropicmarkov})
for all such conditional independencies.
In that language, \emph{the joint distribution of a set of random
variables obeys the Markov property w.r.t.\ to Bayesian network if and only if
its entropy vector lies in the polyhedral convex cone
$\Gamma_n^c:=\Gamma_n \cap \Gamma_c$}, that is, the distribution defines a valid entropy vector (obeying \eqref{shannonineqs_basic}) that is contained in $\Gamma_c$.
The rest of this paper is concerned with the information that can be
extracted from this convex polyhedron.

We remark that this framework can easily be generalized in various
directions. E.g., it is simple to incorporate certain quantitative
bounds on causal influence. Indeed, small deviations of conditional
independence can be expressed as $I(X:Y|Z)\leq \epsilon$ for some
$\epsilon>0$. This is a (non-homogeneous) linear inequality on
$R_n$. One can add any number of such inequalities to the definition
of $\Gamma_n^c$ while still retaining a convex polyhedron (if no
longer a cone). The linear programming algorithm presented below will
be equally applicable to these objects. (In contrast to entropies, the
set of probability distributions subject to quantitative bounds on
various mutual informations seems to be computationally and
analytically intractable).

Another generalization would be to replace Shannon entropies by other,
non-statistical, information measures. To measure similarities of
strings, for instance, one can replace $H$ with Kolmogorov complexity,
which (essentially) also satisfies the polymatroidal axioms
\eqref{shannonineqs_basic}.
Then, the conditional mutual
information
measures conditional algorithmic dependence.
Due to the algorithmic Markov condition, postulated in
\cite{Janzing2010}, causal structures in nature also imply algorithmic
independencies in analogy to the statistical case.
We refer the
reader to Ref. \cite{Steudel_Janzing2010} for further information
measures satisfying the polymatroidal axioms.

\subsection{Marginal Scenarios}

We are mainly interested in situations where not all joint
distributions are accessible. Most commonly, this is because the
variables $X_1, \dots, X_n$ can be divided into observable ones $V_1,
\dots, V_m$ (e.g.\ medical symptoms) and hidden ones $U_1, \dots,
U_{n-m}$ (e.g.\ putative genetic factors). In that case, it is natural
to assume that any subset of observable variables can be
\emph{jointly} observed. There are, however, more subtle situations
(c.f.\ Sec.~\ref{sec:bad_statistics}).
In quantum mechanics, e.g.,
position and momentum of a particle are individually measurable,
as is any combination of position and momentum of two distinct
particles -- however, there is no way to consistently assign a joint
distribution to both position and momentum of the same particle \cite{Bell1964}.

This motivates the following definition: Given a set of variables
$X_{1}, \dots, X_{n}$,
a \emph{marginal scenario} $\mathcal{M}$
is the collection of those subsets of $X_1, \dots, X_n$ that are
assumed to be jointly measurable.

Below, we analyze the Shannon-type inequalities that result
from a given Bayesian network and constrain
the entropies accessible in a marginal scenario $\mathcal{M}$.

\section{Algorithm for the entropic characterization of any DAG}
\label{sec:algorithm}

Given a DAG consisting of $n$ random variables and a marginal scenario
$\mathcal{M}$, the following
steps will produce all
Shannon-type inequalities for the marginals:

\begin{description}
				\item[Step 1] \emph{Construct a description of the unconstrained
				Shannon cone}. 
				This means enumerating all $n+\binom{n}{2} 2^{n-2}$ elementary
				inequalities given in \eqref{shannonineqs_basic}.

		\item[Step 2] \emph{Add causal constraints} presented as in
		(\ref{eqn:entropicmarkov}).
		This corresponds to
		employing the $d$-separation criterion to construct all
		conditional independence relations implied by the DAG.
		
		\item[Step 3]\emph{Marginalization}.
		Lastly, one has to eliminate all joint entropies not contained in
		$\mathcal{M}$.
\end{description}

The first two steps have been described in
Sec.~\ref{sec:info_approach}. We thus briefly discuss the
marginalization, first from a geometric, then from an algorithmic
perspective.

Given a set function $h: 2^{[n]}\to \mathbb{R}$, its restriction
$h_{|\mathcal{M}}:\mathcal{M}\to\mathbb{R}$ is trivial to compute: If
$h$ is expressed as a vector in $R_n$, we just drop all coordinates of
$h$ which are indexed by sets outside of $\mathcal{M}$.
Geometrically, this amounts to a projection $P_\mathcal{M}:
\mathbb{R}^{2^n} \to \mathbb{R}^{|\mathcal{M}|}$. The image of the
constrained cone $\Gamma_n^c$ under the projection $P_\mathcal{M}$ is
again a convex cone, which we will refer to as $\Gamma^{\mathcal{M}}$.
Recall that there are two dual ways of representing a polyhedral
convex cone: in terms of either its extremal rays, or in terms of the
inequalities describing its facets \cite{aliprantis2007cones}.
To determine the projection $\Gamma^{\mathcal{M}}$, a natural
possibility would be to calculate the extremal rays of $\Gamma^{c}_n$
and remove the irrelevant coordinates of each of them. This would
result in a set of rays generating $\Gamma^{\mathcal{M}}$.
However, Steps 1 \& 2 above give a representation of $\Gamma_n^c$ in
terms of inequalities. Also, in order to obtain readily applicable
tests, we would prefer an inequality presentation of
$\Gamma^{\mathcal{M}}$.
Thus, we have chosen an algorithmically more direct (if geometrically
more opaque) procedure by employing Fourier-Motzkin elimination -- a
standard linear programming algorithm for eliminating variables from
systems of inequalities \cite{Williams1986}.

In the remainder of the paper, we will discuss applications of
inequalities resulting from this procedure to causal inference.

\section{Applications}
\label{sec:applications}

\subsection{Conditions for Instrumentality}
\label{subsec:instrumental}
An instrument $Z$ is a random variable that under certain assumptions helps identifying the causal effect of a variable $X$ on another variable $Y$ \cite{Goldberger1972,Pearl1995,Bonet2001}.
The simplest example is given by the instrumentality DAG in Fig.~\ref{fig:DAGS_instrumental} (a), where $Z$ is an instrumental variable and the following independencies are implied: (i) $I(Z:Y\vert X,U)=0$ and
(ii) $I(Z:U)=0$. The variable $U$ represents all possible factors (observed and unobserved) that may effect $X$ and $Y$. Because conditions (i) and (ii) involve an unobservable variable $U$, the use of an instrument Z can only be 
\begin{figure}[t]
\vspace{0.6cm}
\center
\includegraphics[width=7cm]{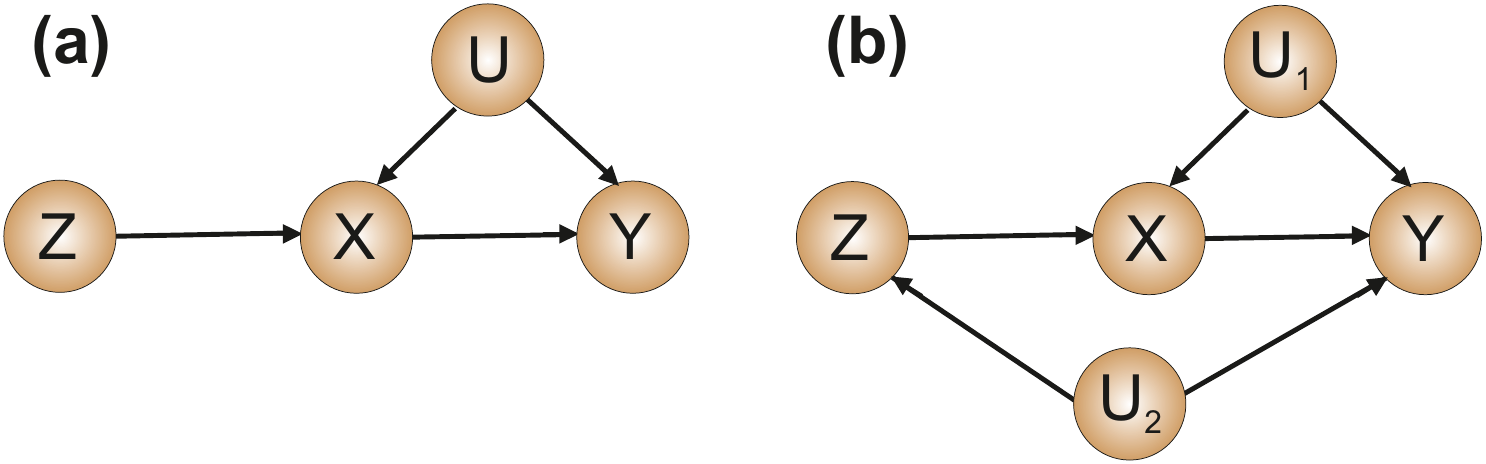}
\caption{DAG (a) represents the instrumental scenario. DAG (b) allows for a common ancestor between $Z$ and $Y$: unless some extra constraint is imposed (e.g.\ $I(Y,U_2) \leq \epsilon$) this DAG is compatible with any probability distribution for the variables $X$, $Y$ and $Z$.}
\label{fig:DAGS_instrumental}
\end{figure}
justified if the observed distribution falls inside the compatibility region implied by the instrumentality DAG.
The distributions compatible with this scenario can be written as
\begin{equation}
\label{pinstrumental}
p(x,y \vert z)=\sum_{u}p(u) p(y \vert x,u) p(x \vert z,u)
\end{equation}
Note that \eqref{pinstrumental} can be seen as a convex combination of
deterministic functions assigning the values of $X$ and $Y$
\cite{Pearl1995,Bonet2001,ramsahai2012causal}. Thus, the region of compatibility associated
with $p(x,y \vert z)$ is a polytope and all the probability
inequalities characterizing it can in principle be determined using
linear programming. However, as the number of values taken by the
variables increases, this approach becomes intractable \cite{Bonet2001}
(see below for further comments).
Moreover, if we allow for variations in the causal relations, e.g.\
the one shown in DAG (b) of Fig. \ref{fig:DAGS_instrumental}, the
compatibility region is not a polytope anymore and computationally
challenging algebraic methods would have to be used \cite{Geiger1999}. For
instance, the quantifier elimination method in \cite{Geiger1999} is
unable to deal with the instrumentality DAG even in the simplest case
of binary variables. We will show next how our framework can easily
circumvent such problems.

Proceeding with the algorithm described in Sec.~\ref{sec:algorithm},
one can see that after marginalizing over the latent variable $U$, the
only non-trivial entropic inequality constraining the instrumental
scenario is given by
\begin{equation}
I(Y:Z\vert X)+I(X:Z) \leq H(X).
\label{instrumental_entropic}
\end{equation}
By ``non-trivial'', we mean that (\ref{instrumental_entropic}) is
not implied by monotonicity and sub-modularity for the observable
variables.
\begin{figure}[t]
\vspace{0.6cm}
\center
\includegraphics[width=5cm]{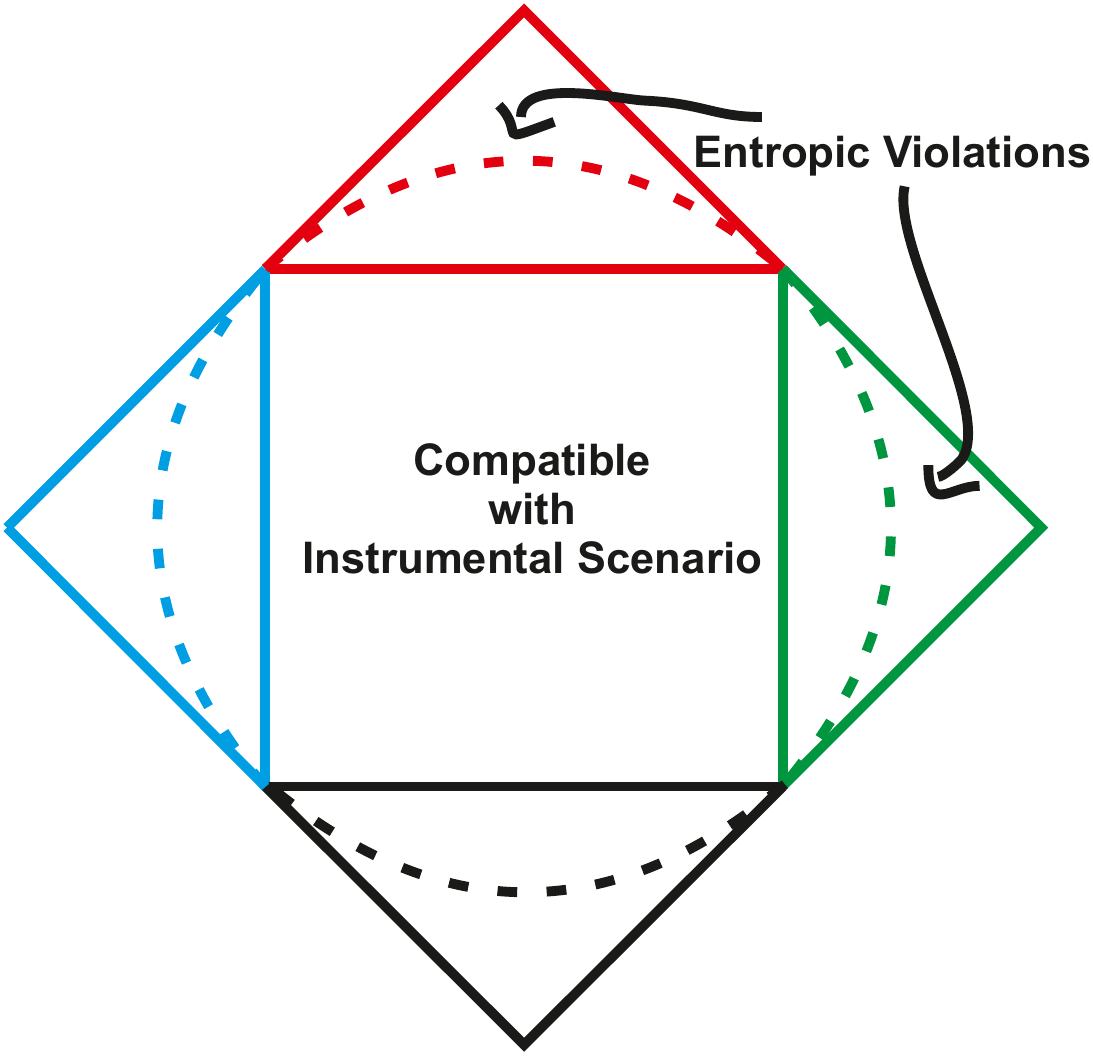}
\caption{A comparison between the entropic and
the probabilistic approach. The squares represent
the polytope of  distributions compatible with the instrumental DAG.
Each facet in the square corresponds to one of the $4$
non-trivial inequalities valid for binary variables
\cite{Pearl1995,Bonet2001}. The triangles over the squares represent probability distributions that fail to be
compatible with the instrumental constraints. Distributions outside
the dashed curve are detected by the
entropic inequality \eqref{instrumental_entropic}.
Due to its
non-linearity in terms of probabilities,
\eqref{instrumental_entropic} detects the non-compatibility associated
with different probability inequalities. See \cite{Chaves2013ns} for more details.}
\label{fig:polytope_instrumental}
\end{figure}
The causal interpretation of \eqref{instrumental_entropic} can be
stated as follows: Since $Z$ influence $Y$ only through $X$, if the
dependency between $X$ and $Z$ is large, then necessarily the
dependency between $Y$ and $Z$ conditioned on knowing $X$ should be
small.

We highlight the fact that, irrespective of how many values the
variables $X$, $Y$ and $Z$ may take (as long as they are discrete),
\eqref{instrumental_entropic} is the only non-trivial entropic
constraint bounding the distributions compatible with the
instrumentality test. This is in stark contrast with the probabilistic
approach, for which the number of linear inequalities increases
exponentially with the number of outcomes of the variables
\cite{Bonet2001}. There is, of course, a price to pay for this concise
description:
There are
distributions that are not compatible with the instrumental
constraints, but fail to violate (\ref{instrumental_entropic}).
In this sense, an entropic inequality is a necessary but
not sufficient criterion for compatibility.
However, it is still surprising that a single entropic inequality can
carry information about causation that is in principle contained only in
exponentially many probabilistic ones.
This effect stems from the non-linear nature of entropy and is illustrated in
Fig.~\ref{fig:polytope_instrumental}.  We remark that
the reduction of descriptional complexity resulting from the use of
non-linear inequalities occurs for other convex bodies as well. The simplest
example along these lines is the Euclidean unit ball $B$. It
requires infinitely many linear inequalities to be defined (namely
$B=\{ x \,|\, (x,y) \leq 1 \forall y, \|y\|_2 \leq 1\}$). These can, of
course, all be subsumed by the single non-linear condition
$\|x\|_2\leq 1$.

Assume now that some given distribution $p(x,y\vert z)$ is
incompatible with the instrumental DAG. That could be due to some
dependencies between $Y$ and $Z$ mediated by a common hidden variable
$U_2$ as shown in DAG (b) of Fig. \ref{fig:DAGS_instrumental}. Clearly,
this DAG can explain any distribution $p(x,y\vert z)$ and therefore is
not very informative. Notwithstanding, with our approach we can for
instance put a quantitative lower bound on how dependent $Y$ and
$U_2$ need to be.
Following the algorithm in Sec.~\ref{sec:algorithm}, one can see that
the only non-trivial constraint on the dependency between $Y$ and
$U_2$ is given by $I(Y:U_2) \leq H(Y \vert X)$. This inequality
imposes a kind of \emph{monogamy of correlations}: if the uncertainty
about $Y$ is small given $X$, their dependency is large, implying
that $Y$ is only slightly correlated with $U_2$, since the latter is
statistically independent of $X$.

\subsection{Inferring direction of causation}
\label{sec:bad_statistics}

As mentioned before, if all variables in the DAG are observed, the
conditional independencies implied by the graphical model completely
characterize the possible probability distributions \cite{Pearl1990}.
For example, the DAGs displayed in Fig.~\ref{fig:bad_stats} display a
different set of CIs. For both DAGs we have $I(X:Z \vert Y,W)=0$,
however for DAG (a), it holds that $I(Y:W \vert X)=0$ while for DAG
(b) $I(Y:W \vert Z)=0$. Hence, if the joint distributions of $(Y,W,X)$
and $(Y,W,Z)$ are accessible, then CI information can distinguish
between the two networks and thus reveal the ``direction of
causation''.

In this section, we will show that the same is possible even if only
two variables are jointly accessible at any time. We feel this is
relevant for three reasons.

First -- and somewhat subjectively -- we believe the insight to be
interesting from a fundamental point of view. Inferring the direction
of causation between two variables is a notoriously thorny issue,
hence it is far from trivial that it can be done from
information about several pairwise distributions.

The second reason is that there are situations where joint
distributions of many variables are unavailable due to practical or
fundamental reasons. We have already mentioned quantum mechanics as
one such example -- and indeed, the present DAGs can be related to
tests for quantum non-locality. We will briefly discuss the details
below.  But also purely classical situations are conceivable. For instance, Mendelian randomization is a good example where the joint distribution on all variables is often unavailable \cite{didelez2007mendelian}.

Thirdly, the ``smoothing effect'' of marginalizing may simplify the
statistical analysis when only few samples are available.
Conditioning on many variables or on variables that attain many
different values often amounts to conditioning on  events that
happened only once. Common $\chi^2$-tests for CI
\cite{upton2000conditional} involve divisions by empirical estimates of
variance, which lead to nonsensical results if no variance is
observed. Testing for CI in those situations requires strong
assumptions (like smoothness of dependencies)
\begin{figure}[t]
\vspace{0.6cm}
\center
\includegraphics[width=6cm]{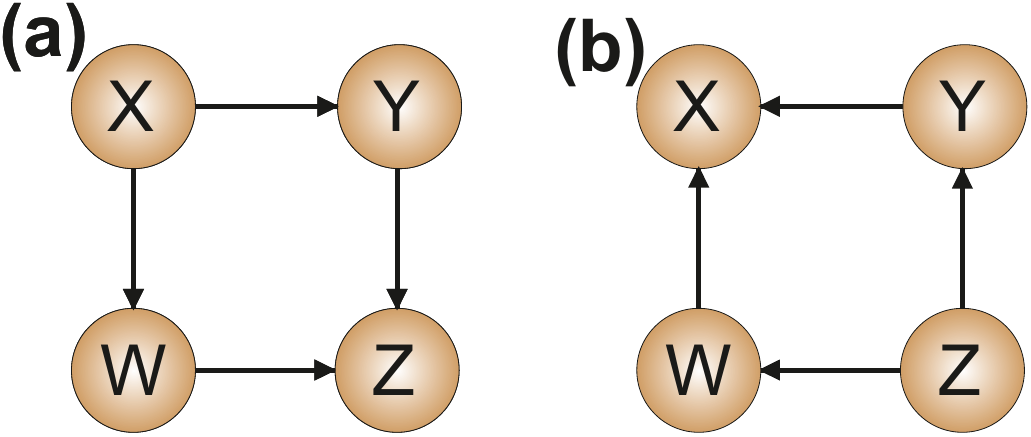}
\caption{DAGs with no hidden variables and opposite causation directions. The DAGs can be distinguished based on the CIs induced by them. However, if only pairwise information is available one must resort to the marginalization procedure described in Sec. \ref{sec:algorithm}.}
\label{fig:bad_stats}
\end{figure}
and remains a challenging research topic
\cite{CondHSIC,UAI_Kun_kernel}. Two-variable marginals, while
containing strictly less information than three-variable ones, show
less fluctuations and might thus be practically easier to handle. This benefit may not sound spectacular as long as it refers to $2$- versus $3$-variable marginals. However, in general, our formalism can provide inequality
constraints for $k$-variable marginals from equality constraints that involve $\ell$-variable marginals for $\ell\gg k$.

We note that causal inference schemes using only pairwise mutual
information is already known for trees, i.e., DAGs containing no
undirected cycles.
The data processing inequality implies that
for every node, the mutual information to a direct neighbor cannot be
smaller than the one with the neighbor of this neighbor. Hence one can
find adjacencies based on pairwise mutual information only.
This has been used e.g.\ for phylogenetic trees
\cite{grumbach94new,ChenCompression}.  In that sense, our results
generalize these ideas to DAGS with cycles.

The non-trivial constraints on two-variable entropies given by our algorithm for the DAG (a) of Fig.~\ref{fig:bad_stats} are:
\begin{align}
\label{nontrivial_pair_4}
& H_{Y}-H_{X}-H_{YW}+H_{XW}\leq 0    \\ \nonumber
& H_{W}-H_{X}-H_{YW}+H_{XY} \leq 0  \\ \nonumber
& H_{WZ}-H_{YW}-H_{XZ}+H_{XY} \leq 0 \\ \nonumber
& H_{YZ}-H_{YW}-H_{XZ}+H_{XW}  \leq 0 \\ \nonumber
& H_{Y}-H_{X}+H_{W}-H_{WZ}-H_{YZ}+H_{XZ} \leq 0 \\ \nonumber
& H_{Z}-H_{X}-H_{YW}-H_{XZ}+H_{XW} +H_{XY} \leq 0 \\ \nonumber
& H_{Z}+H_{X}\\ \nonumber
+&H_{YW}+H_{XZ}-H_{XW}-H_{XY}-H_{WZ}-H_{YZ}  \leq 0.
\end{align}
The ones for DAG (b) are obtained by the substitution $X\leftrightarrow
Z$. Invariant under this, the final inequality is valid for both scenarios.
In contrast, the first six inequalities can be used to distinguish the
DAGs.

As an example, one can consider the following structural equations
compatible only with the DAG (b): $Z$ is a uniformly distributed
$m$-valued random variable, $Y=W=Z$, and $X= Y \oplus W$ (addition
modulo $m$).
%
A direct calculation shows that the first inequality in
\eqref{nontrivial_pair_4} is violated, thus allowing one to infer the
correct direction of the arrows in the DAG.

As alluded to before, we close this section by mentioning a connection
to quantum non-locality \cite{Bell1964}.  Using the linear programming
algorithm, one finds that the final inequality in
(\ref{nontrivial_pair_4}) is actually valid for \emph{any}
distribution of four random variables, not only those that constitute
Bayesian networks w.r.t.\ the DAGs in Fig.~\ref{fig:bad_stats}. In
that sense it seems redundant, or, at best, a sanity check for
consistency of data. It turns out, however, that it can be put to
non-trivial use. While the purpose of causal inference is to check
compatibility of data with a presumed causal structure, the task of
quantum non-locality is to devise tests of compatibility with
classical probability theory as a whole. Thus, if said inequality is
violated in a quantum experiment, it follows that there is no way to
construct a joint distribution of all four variables that is
consistent with the observed two-variable marginals -- and therefore
that classical concepts are insufficient to explain the experiment.

While not every inequality which is valid for all classical
distributions can be violated in quantum experiments, the constraints
in (\ref{nontrivial_pair_4}) do give rise to tests with that property.
%
%
To see this, we further marginalize over $H(X,Z)$ and $H(Y,W)$
to obtain
\begin{equation}
\label{pair_4}
H_{XY}+H_{XW} +H_{YZ} -H_{WZ} -H_{Y}-H_{X}  \leq 0
\end{equation}
(and permutations thereof). These relations have been studied as the
``entropic version of the CHSH Bell inequality''
in the physics literature
\cite{Braunstein1988,FritzChaves2012,Chaves2013b}, where it is shown
that (\ref{pair_4}) can be employed to witness that certain
measurements on quantum systems do not allow for a classical model.
\subsection{Inference of common ancestors in semi-Markovian models}
\label{sec:ancestors}

In this section, we re-visit in greater generality the problem considered in \cite{Steudel2010}: using entropic conditions
to distinguish between hidden common ancestors.

Any distribution of a set of $n$ random variables can
be achieved if there is one latent parent (or
\emph{ancestor}) common to all of them \cite{Pearlbook}.
However, if the dependencies
can also be obtained from a less expressive DAG -- e.g.\ one
where at most two of the observed variables share an ancestor --
then Occam's Razor
would suggest that this model is preferable.  The question is then:
what is the simplest common ancestor causal structure explaining a
given set of observations?

One should note that unless we are able to intervene in the system
under investigation, in general it may be not possible to distinguish
direct causation from a common cause. For instance, consider the DAGs
(a) and (c) displayed in Fig.~\ref{fig:triangle}. Both DAGs are
compatible with any distribution and thus it is not possible to
distinguish between them from passive observations alone. For this
reason and also for simplicity, we restrict our attention to semi-Markovian models where all
the observable variables are assumed to have no direct causation on
each other or on the hidden variables. Also, the hidden variables are
assumed to be mutually
independent. It is clear then that all dependencies between the
observed quantities can only be mediated by their hidden common
ancestors. We refer to such models as common ancestors (CM) DAGs. We reinforce, however, that our framework can also be applied in the most general case.
As will be explained in more details in Sec.~\ref{sec:quant_causal}, in some cases, common causes can be
distinguished from direct causation. Our framework can also be readily
applied in these situations.

We begin by considering the simplest non-trivial case, consisting of
three observed variables \cite{Steudel2010,Fritz2012,Chaves2013b}. If
no conditional independencies  between the variables
occur, then the graphs in
Fig.~\ref{fig:triangle} (a) and (b)
represent the only compatible CM DAGs.
Applying the algorithm described in
Sec. \ref{sec:algorithm} to the model (b), we find that
one
non-trivial class of constraints is given by
\begin{equation}\label{triangle_1}
	I(V_1:V_2)+I(V_1:V_3)     \leq H(V_1)
\end{equation}
and permutations thereof
\cite{Fritz2012,Chaves2013b}.

It is instructive to pause and interpret (\ref{triangle_1}).
It
states, for example, that if the dependency between $V_1$ and $V_2$
is maximal ($I(V_1:V_2)=H(V_1)$) then there should be no dependency at
all
between $V_1$ and $V_3$ ($I(V_1:V_2)=0$). Note that
$I(V_1:V_2)=H(V_1)$ is only possible if $V_1$ is a deterministic
function of the common ancestor $U_{12}$ alone. But if $V_1$ is
independent of $U_{13}$, it cannot depend on $V_3$ and
thus $I(V_1:V_3)=0$.

Consider for instance a distribution given by
\begin{equation}
\label{perf_corr}
p\left(  v_1,v_2,v_3 \right)  =\left\{
\begin{array}{ll}
1/2 & \text{, if } v_1=v_2=v_3\\
0 & \text{, otherwise}%
\end{array}
\right. ,
\end{equation}
This stands for a perfect correlation between all the three variables
and clearly cannot be obtained by pairwise common ancestors. This
incompatibility is detected by the violation of \eqref{triangle_1}.

We now establish the following generalization of \eqref{triangle_1} to an arbitrary number of
observables:
\begin{figure}[t]
\vspace{0.6cm}
\center
\includegraphics[width=8cm]{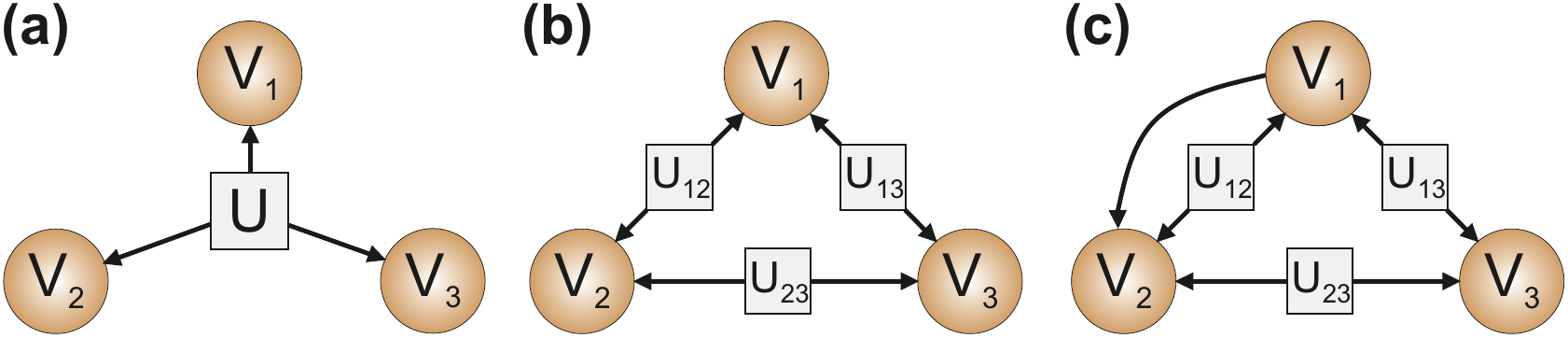}
\caption{
Models (a) and (b) are CM DAGs for three
observable variables $V_1, V_2, V_3$. Unlike (b), DAG (a) is
compatible with any observable distribution. DAG (c) involves a direct causal influence between the observable variable $V_1$ and $V_2$.
}
\label{fig:triangle}
\end{figure}

\begin{theorem}\label{thm:ineq_mn}
For any distribution that can be explained by a CM DAG where each of
the latent ancestors influences at most $m$ of the observed variables,
we have
\begin{equation}
\label{ineq_mn}
\sum_{\substack{
   i=1,\cdots,n \\
   i \neq j
  }}I(V_{i}:V_{j})\leq (m-1)H(V_{j}).
\end{equation}
\end{theorem}
We present the proof for the case $m=2$ while the general proof can be found in the supplemental material.

\begin{lemma}
\label{lemma1}
In the setting of Thm.~\ref{thm:ineq_mn} for $m=2$:
\begin{equation}
 \sum_{i=2}^{n}H(V_{j}U_{ji})\geq
 (N-2)H(V_{j})+H(V_{j}\bigcup_{i=2}^{N}U_{ji}).
\end{equation}
\end{lemma}
\begin{proof}(By induction)
We treat the case $j=1$ w.l.o.g.
For $n=2$ equality holds trivially. Now assuming the validity of the inequality for any $n$:
\begin{eqnarray}
\label{eq2.1}&\sum_{i=2}^{n+1}H(V_{1}U_{1i})\geq (n-2)H(V_{1}) &\\ \nonumber
&+H(V_{1}\bigcup_{i=2}^{n}U_{1i}) +H(V_{1}U_{1(n+1)})&\\
\label{eq2.2}& \geq [(n+1)-2]H(V_{1})+H(V_{1}\bigcup_{i=2}^{n+1}U_{1i}).&
\end{eqnarray}
From (\ref{eq2.1}) to (\ref{eq2.2}) we have used sub-modularity.
\end{proof}

\begin{proof}[Proof of Theorem~\ref{thm:ineq_mn}.]
Apply the data processing inequality to the left-hand side of
\eqref{ineq_mn} to obtain
\begin{eqnarray*}
   &\sum_{i=2}^{n}I(A_{1}:A_{i})\leq \sum_{i=2}^{n}I(A_{1}:U_{1i})&\\ \nonumber
   &=(n-1)H(A_{1})+\sum_{i=2}^{n}H(\lambda_{1i})-\sum_{i=2}^{n}H(A_{1}\lambda_{1i}).&
\end{eqnarray*}

With Lemma~\ref{lemma1}, we get
\begin{eqnarray*}
   &\sum_{i=2}^{n}I(V_{1}:V_{i})\leq(n-1)H(V_{1})+\sum_{i=2}^{n}H(U_{1i})&\\ \nonumber
   &-[(n-2)H(V_{1})+H(V_{1}\bigcup_{i=2}^{n}U_{1i})]&.
\end{eqnarray*}
 The mutual independence of hidden variables yields
 $\sum_{i=2}^{n}H(U_{1i})=H(\bigcup_{i=2}^{n}U_{1i})$ implying that
  \begin{equation*}
		\sum_{i=2}^{n}I(V_{1}:V_{i})\leq
		H(V_{1})-H(V_{1}|\bigcup_{i=2}^{n}U_{1i})\leq H(V_{1}).
 \end{equation*}
\end{proof}

We highlight the fact that Ineq.~\eqref{ineq_mn} involves only
pairwise distributions -- the discussion in
Sec.~\ref{sec:bad_statistics}
applies.
Following our approach, one can derive further entropic inequalities, in
particular involving the joint entropy of all observed variables. A more complete theory will be presented elsewhere.

\subsection{Quantifying causal influences}
\label{sec:quant_causal}

Unlike conditional independence, mutual information captures
dependencies in a quantitative way. In this section, we
show that our framework allows one to derive non-trivial
bounds on the strength of causal links. We then go on to present two
corollaries of this result: First, it follows that the degree of
violation of an entropic inequality often carries an operational
meaning. Second, under some assumptions, the finding will allow us to
introduce a novel way of distinguishing dependence created through
common ancestors from direct causal influence.

Various measures of causal influence have been studied in the
literature. Of particular interest to us is the one recently
introduced in \cite{Janzing2013}. The main idea is that the causal
strength $\mathcal{C}_{X \rightarrow Y}$ between a variable $X$ on
another variable $Y$ should measure the impact of an intervention that
removes the arrow between them. Ref.~\cite{Janzing2013}
draws up a list of reasonable postulates that a measure of causal
strength should fulfill. Of special relevance to our
information-theoretic framework is the axiom stating that
\begin{equation}
	\mathcal{C}_{X \rightarrow Y} \geq I(X:Y \vert \mathrm{PA}_{Y}^{X}),
	\label{causal_influence}
\end{equation}
where $\mathrm{PA}_{Y}^{X}$ stands for the parents of variable
$Y$ other than $X$. We focus on this property, as the quantity $I(X:Y \vert
\mathrm{PA}_{Y}^{X})$ appears naturally in our description and thus allows
us to bound any measure of causal strength $\mathcal{C}_{X \rightarrow
Y}$ for which (\ref{causal_influence}) is valid.

To see how this works in practice, we start by augmenting the common
ancestor scenario considered in the previous section. Assume that now
we do allow for direct causal influence between two variables, in
addition to pairwise common ancestors -- c.f.\ Fig.~\ref{fig:triangle}
(c).
Then (\ref{causal_influence}) becomes $\mathcal{C}_{V_1 \rightarrow V_2} \geq I(V_1:V_2 \vert U_{12},U_{13})$.
We thus re-run our algorithm, this time with the unobservable quantity
$I(V_1:V_2 \vert U_{12},U_{13})$ included in the marginal scenario.
The result is
\begin{equation}
	I(V_1:V_2 \vert U_{12},U_{13})
	\geq
	I(V_1:V_2)+I(V_1:V_3)-H(V_1),
\label{causal_influence_triangle}
\end{equation}
which lower-bounds the causal strength in terms of observable entropies.

The same method yields a particularly concise and relevant result when
applied to the instrumental test of Sec.~\ref{subsec:instrumental}.
The instrumental DAG may stand, for example, for a clinical study
about the efficacy of some drug where $Z$ would label the treatment
assigned, $X$ the treatment received, $Y$ the observed response and
$U$ for any observed or unobserved factors affecting $X$ and $Y$. In
this case we would be interested not only in checking the compatibility
with the presumed causal relations but also the direct causal
influence of the drug on the expected observed response, that is,
$\mathcal{C}_{X \rightarrow Y}$. After the proper marginalization we
conclude that $\mathcal{C}_{X \rightarrow Y} \geq I(Y:Z)$,
a strikingly simple, but non-trivial bound that can be computed from
the observed quantities alone. Likewise,
if one allows the instrumental DAG to
have an arrow connecting $Z$ and $Y$, one
finds
\begin{equation}\label{eqn:instrumental_quant}
	\mathcal{C}_{Z \rightarrow Y} \geq I(Y:Z\vert X)+I(X:Z)- H(X).
\end{equation}

The findings presented here can be re-interpreted in two ways:

First, note that the right hand side of the lower bound
(\ref{causal_influence_triangle}) is nothing but Ineq. \eqref{triangle_1}, a constraint
on distributions compatible with DAG \ref{fig:bad_stats} (b).
Similarly, the r.h.s.\ of (\ref{eqn:instrumental_quant}) is just the
degree of violation of the entropic instrumental inequality
(\ref{instrumental_entropic}).

We thus arrive at the conceptually important realization that the
entropic conditions proposed here offer more than just binary tests.
To the contrary, their degree of violation is seen to carry a
quantitative meaning in terms of strengths of causal influence.

Second, one can interpret the results of this sections as providing a
novel way to distinguish between DAGs (a) and (c) in
Fig.~\ref{fig:triangle} without experimental data.
Assume that we have some information about the physical process that
could facilitate direct causal influence from $V_1$ to $V_2$ in (c),
and that we can use that prior information to put a quantitative upper
bound on $\mathcal{C}_{V_1 \to V_2}$. Then
we must
reject the direct
causation model (c) in favor of a common ancestor explanation (a), as
soon as the observed dependencies violate the bound \eqref{causal_influence_triangle}.
As an illustration, the perfect correlations exhibited by the
distribution
\eqref{perf_corr}
is incompatible with DAG (c), as long as
$\mathcal{C}_{V_1 \rightarrow V_2}$ is known to be smaller than $1$.

\section{Statistical Tests}

In this section, we briefly make the point that inequality-based
criteria immediately suggest test statistics which can be used for
testing hypotheses about causal structures. While a thorough treatment of
statistical issues is the subject of ongoing research \cite{bartolucci2000likelihood,ramsahai2011likelihood}, it should
become plain that the framework allows to derive non-trivial tests in
a simple way.

Consider an inequality	$I := \sum_{S\subset 2^{[n]}} c_S H(S) \leq 0$ for suitable coefficients $c_S$.
Natural candidates for test statistics derived from it would be
$T_I := \sum_S c_S \hat H(S)$ or $T_I':=\frac{T_I}{\sqrt{\hat{\operatorname{var}}(T_I)}}$,
where $\hat H(S)$ is the entropy of the empirical distribution of
$X_S$, and $\hat{\operatorname{var}}$ is some consistent estimator of
variance (e.g.\ a bootstrap estimator).
If the inequality $I$ is fulfilled for some DAG $G$, then a test with
null hypothesis ``data is compatible with $G$'' can be designed by
testing $T_I\leq t$ or $T_I'\leq t$, for some critical value $t>0$.
In an asymptotic regime, there could be reasonable hope to
analytically characterize the distribution of $T_I'$. However, in the
more relevant small sample regime, one will probably have to resort
to Monte Carlo simulations in order to determine $t$ for a desired
confidence level. In that case, we prefer to use $T_I$, by virtue of
being ``less non-linear'' in the data.

We have performed a preliminary numerical study using the DAG given in
Fig.~\ref{fig:triangle} (b) together with Ineq.~(\ref{triangle_1}).
We have simulated experiments that draw $50$ samples from various
distributions of three binary random variables $V_1, V_2, V_3$ and
compute the test statistic $T_I$. To test at the 5\%-level, we must
choose $t$ large enough such that for all distributions $p$
compatible with \ref{fig:triangle}(b), we have a type-I error rate
$\operatorname{Pr}_p[T_I > t]$ below $5\%$.
We have employed the following heuristics for finding $t$: (1) It is
plausible that the highest type-I error rate occurs for distributions
$p$ that
reach equality $\mathbb{E}_p[\hat I]=0$;
(2) This occurs only if $V_1$ is a deterministic function of $V_2$ and
$V_3$. From there, it follows that $V_1$ must be a function of one of
$V_2$ or $V_3$ and we have used a Monte Carlo
simulation with
$(V_2, V_3)$ uniformly random and $V_1=V_2$ to find $t=.0578$.
Numerical checks failed to identify distributions with higher type-I
rate (though we have no proof).  Fig.~\ref{fig:power} illustrates the
resulting test.

\begin{figure}[t]
\vspace{0.8cm}
\center
\includegraphics[width=8cm]{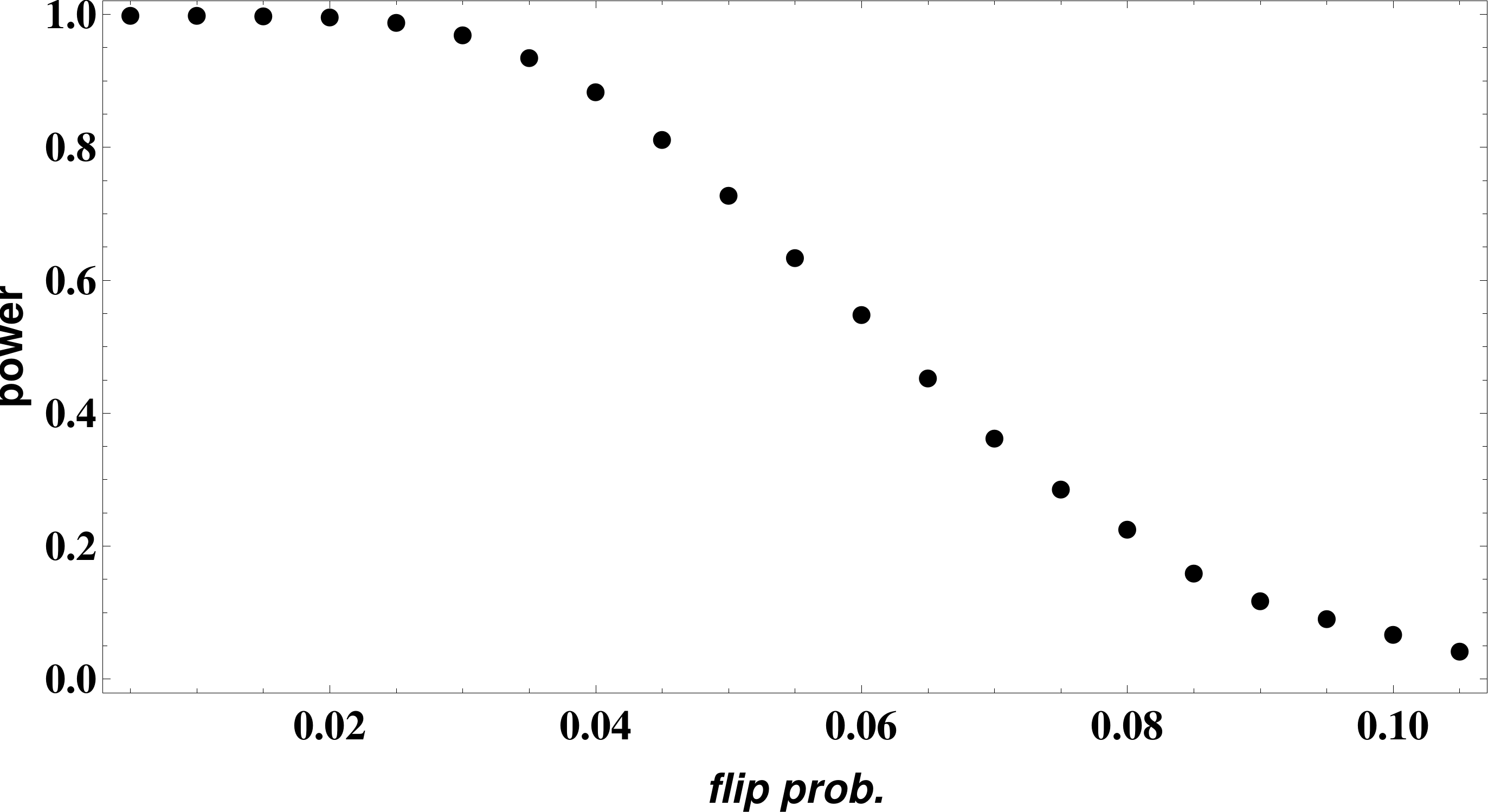}
\caption{
	Power ($1$ minus type-II error) of the test $T_I\geq t$ for the DAG
	Fig.~\ref{fig:triangle}(b) derived from Ineq.~(\ref{triangle_1})
	using $50$ samples.
	The test was run on a distribution obtained by starting with three
	perfectly correlated binary random variables as in (\ref{perf_corr})
	and then inverting each of the variables independently with a given
	``flip probability'' ($x$ axis).
	Every data point is the result of 10000 Monte Carlo simulations.
}
\label{fig:power}
\end{figure}

\section{Conclusions}
Hidden variables imply nontrivial constraints on observable
distributions. While we cannot give a complete characterization of
these constraints, we show that a number of nontrivial constraints can
be elegantly formulated in terms of entropies of subsets of
variables. These constraints are linear (in)equalities, which lend
themselves well to algorithmic implementation.

Remarkably, our approach only requires the polymatroidal axioms, and
thus also applies to various information measures other than Shannon
entropy. Some of these may well be relevant to causal inference and
structure learning and may constitute an interesting topic for future
research.

\section*{Acknowledgements}

We acknowledge support by the Excellence Initiative of the German Federal and State Governments (Grant ZUK 43), the Research Innovation Fund from the University of Freiburg and the Brazilian research
agency CNPq. DG's research is supported by the US Army Research Office under contracts W911NF-14-1-0098 and W911NF-14-1-0133 (Quantum Characterization, Verification, and Validation).

\bibliographystyle{plain}
\bibliography{causal}

\pagebreak
\begin{widetext}
\section{Supplemental Material}
In this supplemental material we prove, for any $n$ and $m$, the validity of the inequality (10) of the main text,
where $n$ is the number of observables and $m$ the maximal number of observables that are connected by one latent ancestors.

Our proof of \textbf{Theorem 1} (inequality (10) of the main text) for general $n$ and $m$ proceeds as follows. We start with Lemma \ref{lemma2}. After some definitions we introduce Lemma \ref{lemma4} which leads to Corollary \ref{corollary1}. This corollary is a statement on how to bound the sum of conditional entropies by another sum of conditional entropies, where the sets over which we condition on are rearranged. Lemma \ref{lemma5} determines which of these sets are empty for CM DAGs with fixed $m$. Finally we connect these results and prove the general inequality.

\begin{lemma}
\label{lemma2}
 For any set of observables $S=({V_{i}\cup V_{j}\cup...})$ and any two (not necessarily) disjoint sets $B_{1},B_{2}$ composed of independent latent ancestors, the following inequality holds
\begin{equation}
 H(S|B_{1})+H(S|B_{2})\geq H(S|B_{1} \cap B_{2}) + H(S|B_{1} \cup B_{2}).
\end{equation}
\end{lemma}

\emph{Proof.}
\begin{eqnarray}
H(S|B_{1})+H(S|B_{2})&=& H(SB_{1})+H(SB_{2})- H(B_{1})-H(B_{2})\\ \nonumber
&\geq& H(S(B_{1} \cap B_{2})) + H(S(B_{1} \cup B_{2}))- H(B_{1})-H(B_{2}).
\end{eqnarray}
Since all latent ancestors are pairwise independent, we have $H(B_{1})+H(B_{2})= H(B_{1} \cap B_{2}) + H(B_{1} \cup B_{2})$ and with this
 \begin{eqnarray}
 &&H(S|B_{1})+H(S|B_{2})\\ \nonumber
 &\geq& H(S(B_{1} \cap B_{2})) + H(S(B_{1} \cup B_{2}))-H(B_{1} \cap B_{2}) - H(B_{1} \cup B_{2})\\ \nonumber
 &=&H(S|B_{1} \cap B_{2}) + H(S| B_{1} \cup B_{2}). \text{  } \square
 \end{eqnarray}

After the following definition we can introduce the next lemma.
\begin{definition}
\label{om}
The latent ancestor connecting the observable variables $V_i$, $V_j$, $V_k$ etc. is labeled $U_{ijk...}$.
We define $A_i$ to be the union of all latent ancestors that connect $V_1$ and $V_i$.
For the case $n=4, m=3$, for example $A_2=\{U_{123},U_{124}\}$
For any scenario with arbitrary, fixed $m$ and $n$, we define
 \begin{equation}
\Omega^{n'}:=\bigcup_{i=2}^{n'} A_{i}\text{   and    } s_{i}^{n'}:=\bigcup_{\substack{j=2\\j\neq i}}^{n'} A_{j},
\end{equation}
where $n'$ is any integer with $n'\leq n$. Additional indices $n$ and $m$, that define the given scenario, are omitted.
\end{definition}
More explicitly, $\Omega^{n'}$ is the union of all sets of latent ancestors up to $n'$; and  $s_{i}^{n'}$ respectively with leaving out $A_{i}$.
To make the definitions clear, we give an explicit example for $n=5$, $m=3$:
\begin{eqnarray}
\Omega^{3}&=&\bigcup_{i=2}^{3} A_{i}\\ \nonumber
&=&A_{2}\cup A_{3}\\ \nonumber
&=&(\lambda_{123}\cup U_{124}\cup U_{125})\cup ( U_{123}\cup U_{134}\cup\ U_{135})\\ \nonumber
&=&U_{123}\cup U_{124}\cup U_{125} \cup U_{134}\cup U_{135}\\
s_{3}^{4}&=&\bigcup_{\substack{j=2\\j\neq 3}}^4 A_{j}=A_{2}\cup A_{4}=...\text{ }.
\end{eqnarray}

 \begin{lemma}
 \label{lemma4}
  With the above definitions the following inequality holds for every $k$ and $n$ with $k\leq n$
  \begin{eqnarray}
  H(V_1|\bigcap_{\substack{s\subseteq [n]\backslash\{1\}\\
  |s|=k-1}}[\cup_{i\epsilon s}A_{i}])
   &+& H(V_1|\bigcap_{\substack{s\subseteq [n]\backslash\{1\}\\
   |s|=k-2}}[\cup_{i\epsilon s}A_{i}\cup A_{n+1}])\\ \nonumber
   \geq H(V_1|\bigcap_{\substack{s\subseteq [n+1]\backslash\{1\}\\ \nonumber
   |s|=k-1}}[\cup_{i\epsilon s}A_{i}])
    &+& H(V_1|\bigcap_{\substack{s\subseteq [n]\backslash\{1\}\\ \nonumber
    |s|=k-2+1}}[\cup_{i\epsilon s}A_{i}\cup A_{n+1}]).
  \end{eqnarray}
 \end{lemma}

 Note that this lemma is not only valid when referring to all $n$ variables but also when replacing $n$ by an integer $n'\leq n$.

 \begin{proof}
 According to Lemma (\ref{lemma2}) we have
  \begin{eqnarray}
 H(V_1|Z)+H(V_1|Y)&\geq& H(V_1|Z\cup Y)+H(V_1|Z\cap Y)\text{ with}\\ \nonumber
   Z&=& \bigcap_{\substack{s\subseteq [n]\backslash\{1\}\\ \nonumber
   |s|=k-1}}[\cup_{i\epsilon s}A_{i}]\text{ and }\\ \nonumber
   Y&=&\bigcap_{\substack{s\subseteq [n]\backslash\{1\}\\ \nonumber
   |s|=k-2}}[\cup_{i\epsilon s}A_{i}\cup A_{n+1}].
  \end{eqnarray}
We calculate the intersection $Z\cap Y$ to be given by
 \begin{eqnarray}
   &&Z\cap Y\\
   \label{eq4}&=& \left(\bigcap_{\substack{s\subseteq [n]\backslash\{1\}\\ \nonumber
   |s|=k-1}}[\cup_{i\epsilon s}A_{i}]\right)\cap\left(
   \bigcap_{\substack{s\subseteq [n]\backslash\{1\}\\ \nonumber
   |s|=k-2}}[\cup_{i\epsilon s}A_{i}\cup A_{n+1}]\right)\\ \nonumber
   &=&\bigcap_{\substack{s\subseteq [n+1]\backslash\{1\}\\|s|=k-1}}[\cup_{i\epsilon s}A_{i}],
\end{eqnarray}
because both sets -- $Z$ and $Y$ -- in (\ref{eq4}) are the intersection of unions of $k-1$ different $A_{i}$, where the $i$ are element of $[n+1]\backslash\{1\}$.
The difference is that $Z$ contains only those unions where $A_{n+1}$ does not appear, $Y$ only those where it does. Subsumed we have the intersection of the
unions of all $k-1$ possible $A_{i}$.

The union can be written as
 \begin{eqnarray}
   &&Z\cup Y\\
   \label{eq5}&=& \left(\bigcap_{\substack{s\subseteq [n]\backslash\{1\}\\ \nonumber
   |s|=k-1}}[\cup_{i\epsilon s}A_{i}]\right)\cup\left(
   \bigcap_{\substack{s\subseteq [n]\backslash\{1\}\\ \nonumber
   |s|=k-2}}[\cup_{i\epsilon s}A_{i}\cup A_{n+1}]\right)\\ \nonumber
   \label{eq5}&=& \left(\bigcap_{\substack{s\subseteq [n]\backslash\{1\}\\ \nonumber
   |s|=k-1}}[\cup_{i\epsilon s}A_{i}]\right)\cup\left(
\bigcap_{\substack{s\subseteq [n]\backslash\{1\}\\ \nonumber
|s|=k-2}}[\cup_{i\epsilon s}A_{i}]\right)\cup A_{n+1}\\ \nonumber
&=& \left(\bigcap_{\substack{s\subseteq [n]\backslash\{1\}\\ \nonumber
|s|=k-1}}[\cup_{i\epsilon s}A_{i}]\right)\cup A_{n+1}\\ \nonumber
&=&\bigcap_{\substack{s\subseteq [n]\backslash\{1\}\\ \nonumber
|s|=k-2+1}}[\cup_{i\epsilon s}A_{i}\cup A_{n+1}],\\ \nonumber
\end{eqnarray}
which concludes the proof.
\end{proof}

\begin{corollary}
 \label{corollary1}
 For every $n'\leq n$ the following inequality is valid
  \begin{equation}
   \sum_{i=2}^{n'}H(V_1|A_{i})\geq\sum_{k=2}^{n'}H(V_1|\bigcap_{\substack{s\subseteq [n']\backslash\{1\}\\
   |s|=k-1}}[\cup_{j\epsilon s}A_{j}]).
  \end{equation}
\end{corollary}
\begin{proof}(By induction)

 For $n'=2$ we have equality. Now we have to show that
 \begin{equation}
 \label{desired}\sum_{i=2}^{n'+1}H(V_1|A_{i})\geq\sum_{k=2}^{n'+1}H(V_1|\bigcap_{\substack{s\subseteq [n'+1]\backslash\{1\}\\
 |s|=k-1}}[\cup_{j\epsilon s}A_{j}]).
 \end{equation}
 Assuming validity of Corollary \ref{corollary1} for $n'$, we get
  \begin{eqnarray}
   &&\sum_{i=2}^{n'+1}H(V_1|A_{i})\\
   &\geq&\sum_{k=2}^{n'}H(V_1|\bigcap_{\substack{s\subseteq [n']\backslash\{1\}\\ \nonumber
   |s|=k-1}}[\cup_{j\epsilon s}A_{j}])+H(V_1|A_{n'+1})\\
   &=&\sum_{k=2}^{n'}H(V_1|\bigcap_{\substack{s\subseteq [n']\backslash\{1\}\\ \nonumber
   |s|=k-1}}[\cup_{j\epsilon s}A_{j}])
   + H(V_1|\bigcap_{\substack{s\subseteq [n']\backslash\{1\}\\ \nonumber
   |s|=0}}[\cup_{i\epsilon s}A_{i}\cup A_{n'+1}])\\
    &=&\sum_{k=3}^{n'}H(V_1|\bigcap_{\substack{s\subseteq [n']\backslash\{1\}\\ \nonumber
    |s|=k-1}}[\cup_{j\epsilon s}A_{j}]) +H(V_1|\bigcap_{\substack{s\subseteq [n']\backslash\{1\}\\
   |s|=1}}[\cup_{j\epsilon s}A_{j}])
   + H(V_1|\bigcap_{\substack{s\subseteq [n']\backslash\{1\}\\ |s|=0}}[\cup_{i\epsilon s}A_{i}\cup A_{n'+1}]).
  \end{eqnarray}
  Now we use Lemma (\ref{lemma4}) to bound the last two terms and get
   \begin{eqnarray}
   &&\sum_{i=2}^{n'+1}H(V_1|A_{i})\\ \nonumber
   &\geq&\sum_{k=3}^{n'}H(V_1|\bigcap_{\substack{s\subseteq [n']\backslash\{1\}\\ \nonumber
   |s|=k-1}}[\cup_{j\epsilon s}A_{j}])+H(V_1|\bigcap_{\substack{s\subseteq [n'+1]\backslash\{1\}\\ \nonumber
    |s|=1}}[\cup_{i\epsilon s}A_{i}])+ H(V_1|\bigcap_{\substack{s\subseteq [n']\backslash\{1\}\\ \nonumber
   |s|=1}}[\cup_{i\epsilon s}A_{i}\cup A_{n'+1}]).
  \end{eqnarray}
  We notice that the second term in RHS is the term $k=2$ of the desired sum in (\ref{desired}). The $k=3$ term of the sum can again be connected to the last term in RHS
  to generate the next term of the desired sum. Repeating this application of Lemma (\ref{lemma4}) we can turn every term of the sum into the desired one.
\end{proof}
Now that we have shown how to rearrange a sum of conditional entropies we examine the sets over which we condition on. We show that some of them can be identified with $\emptyset$
and some with $\Omega^{n}$.
It is the last small step to take, before we can introduce the inequality for general $n$ and $m$.
\begin{lemma}
 \label{lemma5}
   For every $n$ and $m$ with $n\geq m$ and integer $c$ with $c\leq n$, the following holds:
  \begin{eqnarray*}
           \bigcap_{\substack{s\subseteq [n]\backslash\{1\}\\|s|=c}}[\cup_{j\epsilon s}A_{j}]=\begin{cases}
                                                                                               \emptyset &\mbox{if } c \leq n-m \\
                                                                                               \Omega^n &\mbox{if } c > n-m
                                                                                              \end{cases}.
  \end{eqnarray*}
 \end{lemma}
\begin{proof}
 We start with the first case. As the left-hand term is invariant under permutations of the indices $2,...,n$ it can either be $\Omega^n$ or $\emptyset$. So we only have to prove
 \begin{equation}
 \bigcap_{\substack{s\subseteq [n]\backslash\{1\}\\|s|=c}}[\cup_{j\epsilon s}A_{j}]\neq \Omega^n,
  \end{equation}
   which is equivalent to
\begin{eqnarray*}
  \bigcup_{\substack{s\subseteq [n]\backslash\{1\}\\|s|=c}}[\cup_{j\epsilon s}A_{j}]^{C}&\neq& \emptyset\\
  \text{and }\bigcup_{\substack{s\subseteq [n]\backslash\{1\}\\|s|=c}}[\cap_{j\epsilon s}A_{j}^{C}]&\neq& \emptyset.
 \end{eqnarray*}
It is sufficient to present one $s\subseteq [n]\backslash\{1\}$ with $|s|= n-m$ such that $\cap_{j\epsilon s}A_{j}^{C}\neq \emptyset$.

We take $s=[n-m+1]\backslash\{1\}$ and get
  \begin{equation*}
   \bigcap_{j\epsilon s}A_{j}^{C}=\bigcap_{j=2}^{n-m+1}A_{j}^{C}=:Z.
  \end{equation*}
We highlight that $A_{j}^{C}$ is the set of all latent ancestors that are connected to $V_1$ but not to $V_{j}$. So $Z$ is the set of all latent ancestors
$U_{klm...}$ that contain none of the indices in $[n-m+1]\backslash\{1\}$. More precisely, it is the set that consists  only of $U_{1,(n-m+2),...,n}$ (because these are the
remaining $m$ indices).
It follows that $ \bigcap_{j\epsilon s}A_{j}^{C}=U_{1,(n-m+2),...,n}$ and the first part of the lemma is proven. The proof of the second part works equivalently.
\end{proof}

We are now ready to prove \textbf{Theorem 1} (inequality (10)) of the main text.
\begin{theorem}
\label{main}
 For any data that can be explained by a CM DAG where every latent ancestors has at most $m$ children, the inequality
\begin{equation}
 \sum_{i=2}^{n}I(V_1:V_{i})\leq (m-1)H(V_1)
\end{equation}
holds.
 \end{theorem}
 \begin{proof}
  We start with the left-hand side, use the data processing inequality, apply Corollary (\ref{corollary1}) and Lemma (\ref{lemma5}) and bound again, to get
 \begin{eqnarray}
 &&\sum_{i=2}^{n}I(V_1:V_{i})\\ \nonumber
 &\leq&\sum_{i=2}^{n}I(V_1:A_{i})\\ \nonumber
 &\leq&(n-1)H(V_1)-\sum_{i=2}^{n}H(V_1|A_{i})\\ \nonumber
 &\leq&(n-1)H(V_1)-\sum_{k=2}^{n}H(V_1|\bigcap_{\substack{s\subseteq [n]\backslash\{1\}\\|s|=k-1}}[\cup_{j\epsilon s}A_{j}])\\ \nonumber
 &=&(n-1)H(V_1)-(m-1)H(V_1|\Omega)-(n-m)H(V_1)\\ \nonumber
 &\leq& (m-1)H(V_1).
\end{eqnarray}

 \end{proof}

The labeling of the variables is arbitrary but the inequalities are not symmetric under change of indices.
Changing the observable called $V_{1}$ will lead to a different inequality.
Therefore for every $m$ we get $n$ different inequalities.

\end{widetext}

\end{document}